\documentclass{article}
\usepackage{amsmath,amssymb,graphicx,mlspconf}
\usepackage{color,xcolor}
\usepackage[ruled,vlined]{algorithm2e}
\usepackage{booktabs}
\usepackage{multirow}

\newtheorem{myTheo}{Theorem}

\newtheorem{myDef}{Definition}

\newtheorem{proof}{Proof}

\newtheorem{lem}{Lemma}

\title{Wasserstein Nonnegative Tensor Factorization with Manifold Regularization}
\name{Jianyu Wang, Linruize Tang}
\address{Northwestern Polytechnical University}

\begin{document}

\maketitle

\begin{abstract}
Nonnegative tensor factorization (NTF) has become an important tool for feature extraction and part-based representation with preserved intrinsic structure information from nonnegative high-order data.
However, the original NTF methods utilize Euclidean or Kullback-Leibler divergence as the loss function which treats each feature equally {leading to the neglect of} the side-information of features.
To utilize correlation information of features and manifold information of samples, we introduce Wasserstein manifold nonnegative tensor factorization (WMNTF), which minimizes the Wasserstein distance between the distribution of input tensorial data and the distribution of reconstruction.
{Although some researches about Wasserstein distance have been proposed in nonnegative matrix factorization (NMF), they ignore the spatial structure information of higher-order data.}
We use Wasserstein distance (\emph{a.k.a.} Earth Mover's distance or Optimal Transport distance) as \textcolor{red}{a} metric and add a graph regularizer to a latent factor.
Experimental results demonstrate the effectiveness of the proposed method compared with other NMF and NTF methods.
\end{abstract}
\section{Introduction}
\label{sec:intro}

Exploring structure information and compact representation of high-dimensional data is an essential in machine learning and signal processing.
The nonnegative tensor factorization (NTF) provides an efficient way to analyze high-order data \cite{cichocki2009nonnegative} and it has become increasingly important in various applications including hyperspectral image processing \cite{gao2021using}, multichannel blind source separation \cite{mitsufuji2020multichannel}, text mining \cite{nakatsuji2017semantic} and electroencephalography based brain-computer interfaces \cite{li2018scalable}.
The canonical polyadic decomposition (CPD) \cite{shi2017tensor} and the Tucker \cite{shi2018feature} decomposition are the most widely used tensor factorization methods.
In this paper, we focus on the nonnegative CPD (NCPD).
Existing tensor-based methods fit different distributions on data or noise such as Gaussian \cite{cheng2020learning}, Poisson \cite{schein2016bayesian} and Gamma \cite{xu2013bayesian} distributions, and minimize loss functions including Euclidean distance \cite{zhou2012fast}, KL-divergence \cite{hansen2015newton} and other divergences \cite{cichocki2007non}.
These fitting errors are separable and treat each feature equally, consequently, they ignore the similarities the features may share \cite{rolet2016fast}.
{In addition, the Wasserstein distance based matrix factorization methods ignore correlation relations with each tensor mode leading to the destruction of spatial information about higher-order data.
What's more,} another shortcoming of original NMF and NTF is that they neglect the intrinsic manifold structure information of high-dimensional data \cite{cai2010graph}.

Recent works of Wasserstein distance (\emph{a.k.a.} Earth Mover's distance or Optimal Transport distance) \cite{rubner2000earth, cuturi2013sinkhorn, frogner2015learning} have focused on using a better metric to measure the difference between two distributions.
This distance measures the cheapest way to transport the mass from one probability measure to match that in another, with a given cost function.
Wasserstein distance has been successfully applied to matrix or tensor factorization and dictionary learning.
For example, Sandler \emph{et al.} \cite{sandler2011nonnegative} first introduced the earth mover's distance as a substitute for the original loss functions.
Nevertheless, the linear programming steps of each iteration require costly computational resources.
Flamary \emph{et al.} \cite{flamary2016optimal} utilized optimal transportation into music transcription systems and incorporated sparse regularization of activation matrix to avoid disproportional damage.
Qian \emph{et al.} \cite{qian2016non} proposed Wasserstein NMF algorithm with graph regularization of bases matrix to preserve local geometry structure.
Xu \emph{et al.} \cite{xu2018multi} incorporated the Mahalanobis distance as ground distance and smoothed Wasserstein distance to characterize the errors between any two samples.
Afshar \emph{et al.} \cite{afshar2021swift} proposed Wasserstein tensor distance and utilized tensor CP factorization to obtain the reconstructions.
Zhang \cite{zhang2021unified} presented a generalized mathematical framework to compute NMF and NTF with respect to an optimal transportation loss, and derived the solutions by convex dual formulation.
However, the manifold structure information of sample space is ignored in tensor based methods.

We propose a graph regularized Wasserstein nonnegative tensor factorization (GWNTF) method.
{Our objective is to learn the latent low rank representations of high dimensional tensorial data which takes the correlation information across dimensions of the input tensor into consideration and preserve the local geometry structure.
To achieve this, we specify a smooth regularized Wasserstein distance as a metric over original tensorial as well as reconstruction, and includes the graph regularizer.
All the model parameters are derived by an efficient Maximization-Minimization (MM) \cite{fevotte2011algorithms} algorithms.}
The part-based low-rank representations can be efficiently computed.
The experimental results illustrate the effectiveness of the GWNTF.

The brief review of Wasserstein distance and NTF will be introduced in Section \ref{section2}. We introduce our GWNTF model and the optimization of parameters in Section \ref{section3}. Experiments on COIL20 and PIE\_pose27 data set are presented in Section \ref{section4}. Finally, we conclude in Section \ref{section5}.

\section{Preliminary and Problem formulation}
\label{section2}

\subsection{Notations}
Here, we use bold lowercase $\mathbf{x}$, bold uppercase $\mathbf{X}$ and calligraphic letters $\mathcal{X}$ to denote vectors, matrices and tensors, respectively.
$\langle \cdot, \cdot \rangle$ denotes the usual dot product between $x$ and $y$.
Symbols $\circ$, $\otimes$, $\circledast$ and $\odot$ denote the outer, Kronecker, Hadamard and Khatri-Rao products, respectively.
$\oslash$ denotes the element-wise division.
$\mathbf{x}_i$ and $x_{ij}$ denotes the $i$-th column vector and $(i,j)$-th entry of matrix $\mathbf{X}$, respectively.
$\mathbb{R}_{\geq 0}$ denotes the fields of nonnegative real numbers.

\subsection{Wasserstein Distance}
We will discuss the problem of finding low rank representations of NTF under Wasserstein distance and manifold structure of samples in this section.

Generally, Wasserstein distance investigates a way to compare two probability measures over a domain $\mathcal{D}$.
Given a cost function $c$ between points on $\mathcal{D}$, we aim to find the optimal path to transport the point $x$ to $y$. Let $U(x,y)$ be the set of all bivariate probability measures on the product space $\mathcal{D} \times \mathcal{D}$ where the two margin probabilistic distributions are $x$ and $y$, respectively.
The Wasserstein distance can be defined as
\begin{equation}\label{WD}
\begin{split}
 W(x,y)=\inf_{\tau\in U(x,y)} \int c(x,y) d\tau(x,y)
\end{split}
\end{equation}
where the cost function $c(x,y)$ can be represented as a proper metric function for two points. In this paper, we use the discrete description of Wasserstein distance to the real world applications.
Wasserstein distance is the optimal transportation or earth mover's distance problem which computes the distance between two probabilistic vectors $\mathbf{a} \in \mathbb{R}_{\geq 0}^{m}$ and $\mathbf{b} \in \mathbb{R}_{\geq 0}^{m}$.
Given the cost matrix $\mathbf{C} \in \mathbb{R}_{\geq 0}^{m \times m}$, the exact Wasserstein distance between $\mathbf{a}$ and $\mathbf{b}$ can be defined as:
\begin{equation}\label{Wasserstein}
\begin{split}
  W(\mathbf{a},\mathbf{b}) & = \min_{\mathbf{T}\in U(\mathbf{a}, \mathbf{b})} \langle \mathbf{C}, \mathbf{T} \rangle, \\
 \mbox{s.t.} \! \quad \! U(\mathbf{a},\mathbf{b}) \! & = \! \{ \mathbf{T} | \mathbf{T} \in \mathbb{R}_{\geq 0}^{m \times m}, \mathbf{T}\mathbf{1}_{m} \! = \! \mathbf{a}, \mathbf{T}^T \mathbf{1}_m \! = \! \mathbf{b} \} \! , \!
\end{split}
\end{equation}
where $\mathbf{C} = \{ d(i,j) \}, i = 1,\dots,m, j = 1,\dots,m$ denotes the ground distance matrix, and the ground truth metric $d(i,j)$ defines the cost of moving one unit of feature from the source $\mathbf{a}$ to the target $\mathbf{b}$.
$\mathbf{T} = \{ T(i,j) \}, i = 1,\dots,m, j = 1,\dots,m$ is the flow matrix, and $T(i,j)$ denotes the mass of earth moved from the source $\mathbf{a}$ to the target $\mathbf{b}$.

However, the computation of exact Wasserstein distance is time consuming.
To avoid this issue, Cuturi \emph{et al.} \cite{cuturi2013sinkhorn} introduced a smooth Wasserstein distance with an entropic term:
\begin{equation}\label{smoothWasserstein}
\begin{split}
  W(\mathbf{a},\mathbf{b}) & = \min_{\mathbf{T}\in U(\mathbf{a}, \mathbf{b})} \langle \mathbf{C}, \mathbf{T} \rangle - \frac{1}{\lambda} H(\mathbf{T}), \\
 \mbox{s.t.} \! \quad \! U(\mathbf{a},\mathbf{b}) \! & = \! \{ \mathbf{T} | \mathbf{T} \in \mathbb{R}_{\geq 0}^{m \times m}, \mathbf{T}\mathbf{1}_{m} \! = \! \mathbf{a}, \mathbf{T}^T \mathbf{1}_m \! = \! \mathbf{b} \} \! , \!
\end{split}
\end{equation}
where $H(\mathbf{T}) = - \langle \mathbf{T}, \log\mathbf{T} \rangle$ denotes the entropy function with respect to $\mathbf{T}$, and $\lambda$ is hyperparameters of the smooth Wasserstein distance.

To solve the Wasserstein distance of tensor data, we introduce the definition of Wasserstein matrix distance and Wasserstein tensor distance.
\begin{myDef}{(Wasserstein Matrix Distance \cite{afshar2021swift}):}
Given the cost matrix $\mathbf{C} \in \mathbb{R}_{\geq 0}^{M \times M}$, the Wasserstein distance between the source matrix $\mathbf{A} \in \mathbb{R}_{\geq 0}^{M \times N}$ and the target matrix $\mathbf{B} \in \mathbb{R}_{\geq 0}^{M \times N}$ can be represented as:
\begin{equation}\label{WMD}
\begin{split}
  W_M(\mathbf{A},\mathbf{B}) & = \min_{\mathbf{T}_n \in U(\mathbf{A}, \mathbf{B})} \sum_{n=1}^N \langle \mathbf{C}, \mathbf{T}_n \rangle - \frac{1}{\lambda} H(\mathbf{T}_n), \\
 \mbox{s.t.}  U\!(\!\mathbf{A},\!\mathbf{B}) \! = \! &  \{ \mathbf{T}_{\!n\!} \in \mathbb{R}_{\geq 0}^{\!M \! \times \! M \!}, \forall n | \mathbf{T}_{\!n\!} \mathbf{1}_{\!M\!} \! = \! \mathbf{a}_{n\!}, \mathbf{T}_{\!n\!}^{\!T} \! \mathbf{1}_{\!M\!} \! = \! \mathbf{b}_n \} \! . \!
\end{split}
\end{equation}
\end{myDef}
\begin{myDef}{(Wasserstein Tensor Distance \cite{afshar2021swift}):}
Give the source tensor $\mathcal{X} \in \mathbb{R}^{I_1 \times \dots \times I_N}_{\geq 0}$, the target tensor $\mathcal{Y} \in \mathbb{R}^{I_1\times \dots \times I_N}_{\geq 0}$, and the cost tensor $\mathcal{C} \in \mathbb{R}^{I_1 \times \dots \times I_N}_{\geq 0}$, the Wasserstein tenor distance can be denoted as:
\begin{equation}\label{WTD}
\begin{split}
  & W_T(\mathcal{X},\mathcal{Y}) = \min_{\mathcal{T}\in U(\mathcal{X}, \mathcal{Y})} \langle \mathcal{C}, \mathcal{T} \rangle - \frac{1}{\lambda} H(\mathcal{T}) \\
  & \! = \! \sum_{n = 1}^N \left\{ \min_{\mathbf{T}_{n} \in U(\mathbf{X}_{(n)}, \mathbf{Y}_{(n)})} \langle \mathbf{X}_{(n)}, \mathbf{Y}_{(n)} \rangle - \frac{1}{\lambda}H(\mathbf{T}_{n}) \right\} \\
  & \mbox{s.t.} \! \quad \! U(\mathbf{X}_{(n)}, \mathbf{Y}_{(n)}) \! = \! \Big\{ \mathbf{T}_{n} | \mathbf{T}_{n} \in \mathbb{R}_{\geq 0}^{I_n \times I_n I_{-n}}, \\
  & \quad \quad \quad \quad \quad \quad \Phi(\mathbf{T}_{n}) = \mathbf{X}_{(n)}, \Psi(\mathbf{T}_{n}) = \mathbf{Y}_{(n)} \Big\},
\end{split}
\end{equation}
where $\Phi(\mathbf{T}_{n})=\left[ \mathbf{T}_{n}^1 \mathbf{1}_{I_n}, \dots, \mathbf{T}_{n}^{I_{(-n)}}\mathbf{1}_{I_n} \right]$, $\Psi(\mathbf{T}_{n}) = \Big[ \mathbf{T}_{n}^{1^T}\mathbf{1}_{I_n},$ $\dots, \mathbf{T}_{n}^{I_{(-n)}^T}\mathbf{1}_{I_n} \Big]$, the mode-$n$ matricization of the source tensor is $\mathbf{X}_{(n)} \in \mathbb{R}_{\geq 0}^{I_n \times I_{-n}}$, the mode-$n$ matricization of the target tensor is $\mathbf{Y}_{(n)} \in \mathbb{R}_{\geq 0}^{I_n \times I_{-n}}$.
\end{myDef}

\subsection{Objective function}
In order to understand the objective function of the GWNTF greatly, we introduce some definitions of tensor decomposition and Wasserstein distance.
\begin{myDef}{(Matricization of a tensor \cite{kolda2009tensor}):}
It is also known as unfolding or flattening, which is the process of reordering the elements of an $N$-way array into a matrix along each mode. A mode-$n$ matricization of a tensor $\mathcal{X}\in \mathbb{R}^{I_1\times\dots\times I_N}$ is denoted as $\mathbf{X}_{(n)} \in \mathbb{R}^{I_n \times I_1 \dots I_{n-1} I_{n+1}\dots I_N}$ and arranges the mode-$n$ fibers to be the columns of the resulting matrix.
\end{myDef}

\begin{myDef}{(Nonnegative Tensor Factorization \cite{harshman1970foundations}):}
The CP-based tensor model decomposes $\mathcal{X}$ into a linear combination of $R$ rank-one tensors as follows:
\begin{equation}\label{CPD}
\begin{split}
 \mathcal{X} = & \sum_{r=1}^R \mathbf{a}^{(1)}_r \circ \mathbf{a}^{(2)}_r \circ \dots \circ \mathbf{a}^{(N)}_r  = [\![ \mathbf{A}^{(1)},\dots,\mathbf{A}^{(N)} ]\!], \\
 & \mbox{subject to} \quad \mathbf{A}^{(n)} \geq 0, \quad \quad \forall \quad n = 1,\dots, N,
\end{split}
\end{equation}
where $R$ is the CP rank \cite{kolda2009tensor}, $[\![\cdots]\!]$ denotes the shorthand term of Kruskal operator. $\{\mathbf{A}^{(n)} \in \mathbb{R}_{\geq 0}^{I_n \times R}\}_{n=1}^N$ are a set of latent factor matrices, and it can be represented as:
\begin{equation}\label{factormatrices}
\begin{split}
 \mathbf{A}^{(n)} \! = \! \left[ \mathbf{a}_1^{(n)}, \dots, \mathbf{a}_{i_n}^{(n)}, \dots, \mathbf{a}_{I_N}^{(n)} \right]^T \! = \! \left[ \mathbf{a}_{\cdot1}^{(n)}, \dots, \mathbf{a}_{\cdot R}^{(n)} \right] \! .
\end{split}
\end{equation}
\end{myDef}

Similar to the prior works on vector, matrix and tensor based Wasserstein distance \cite{frogner2015learning, qian2016non, afshar2021swift}, we replace the equality constraints on the source marginal distributions and the target marginal distribution with KL divergence, and obtain an unconstrained Wasserstein tensor distance:
\begin{equation}\label{WTD}
\begin{split}
  & W_T(\mathcal{X},\mathcal{Y}) = \min_{\mathcal{T}\in U(\mathcal{X}, \mathcal{Y})} \langle \mathcal{C}, \mathcal{T} \rangle - \frac{1}{\lambda} H(\mathcal{T}) \\
   & + \alpha D_{\mbox{KL}}\left(\Phi(\mathbf{T}_{n}),\mathbf{X}_{(n)}\right) + \beta D_{\mbox{KL}}\left( \Psi(\mathbf{T}_{n}), \mathbf{Y}_{(n)} \right),
\end{split}
\end{equation}
where $D_{\mbox{KL}}\left( \mathbf{X}, \mathbf{Y} \right) = \sum_{i,j} \mathbf{X}_{ij}\log\frac{\mathbf{X}_{ij}}{\mathbf{Y}_{ij}} - \mathbf{X}_{ij} + \mathbf{Y}_{ij}$.
Then, combining the definition of Wasserstein tensor distance, the objective function of GWNTF can be represented as:
\begin{equation}\label{GWNTF}
\begin{split}
  & W_{\!T\!}(\mathcal{X}\!,\!\hat{\mathcal{X}}) \! = \! \sum_{n = 1}^N \! \left\{ \! \min_{\!\mathbf{T}_{\!n\!} \in \! U(\mathbf{X}_{\!(\!n\!)\!}, \hat{\mathbf{X}}_{\!(\!n\!)\!})} \langle \mathbf{T}_{\!n\!}, {\mathbf{C}}_{\!n\!} \rangle \! - \! \frac{1}{\lambda}H(\mathbf{T}_{\!n\!}) \! \right\} \\
   & + \! \alpha D_{\!\mbox{KL}\!}\!\left(\Phi(\mathbf{T}_{\!n\!}),\! \mathbf{X}_{\!(\!n\!)\!}\right) \! \! + \! \! \beta D_{\!\mbox{KL}\!}\!\left( \! \Psi(\mathbf{T}_{\!n\!}), \! \hat{\mathbf{X}}_{\!(\!n\!)\!} \right) \!\! + \! \! \mu \! \mathcal{R} \! \left( \! \hat{\mathcal{X}} \right)\!, \\
   & \mbox{s.t.} \quad \hat{\mathbf{X}}_{\!(\!n\!)\!} \! = \! \mathbf{A}^{\!(\!n\!)\!} \! \left( \! \mathbf{A}^{\!(\!1\!)\!}\!\odot\!\dots\!\odot\! \mathbf{A}^{\!(\!n\!-\!1\!)\!} \! \odot \! \mathbf{A}^{\!(\!n\!+\!1\!)\!}\! \odot\! \dots \! \odot \! \mathbf{A}^{\!(\!N\!)\!} \! \right)^T \! , \\
   & \quad \quad \mathcal{R} \left( \hat{\mathcal{X}} \right) = \sum_{i_N,j_N} \mathbf{V}_{i_Nj_N} \left\| \mathbf{a}^{(N)}_{i_N} - \mathbf{a}^{(N)}_{j_N} \right\|^2,
\end{split}
\end{equation}
where $\Phi(\mathbf{T}_{n})=\left[ \mathbf{T}_{n}^1 \mathbf{1}_{I_n}, \dots, \mathbf{T}_{n}^{I_{(-n)}}\mathbf{1}_{I_n} \right]$, $\Psi(\mathbf{T}_{n}) = \Big[ \mathbf{T}_{n}^{1^T}\mathbf{1}_{I_n},$ $\dots, \mathbf{T}_{n}^{I_{(-n)}^T}\mathbf{1}_{I_n} \Big]$, and $\mathbf{V}_{i_Nj_N}
\in \mathbb{R}_{\geq 0}^{I_N\times I_N}$ denotes a nearest neighbor graph on a sample space.
For convenience, we define $\mathbf{A}^{(-n)} = (  \mathbf{A}^{(1)} \odot \dots \odot \mathbf{A}^{(n-1)} \odot \mathbf{A}^{(n+1)} \odot \dots  \odot \mathbf{A}^{(N)} )$, and $i_{-n} = 1,\dots, I_{-n}$, and $I_{-n} = I_1 \dots I_{n-1} I_{n+1} \dots I_N$.

\section{Optimization of GWNTF}
\label{section3}

{In this section, we first recall the derivation of auxiliary functions for the objective function of GMNTF. Then we derive the update rules for transport tensor $\mathcal{T}$ and latent factors $\mathbf{A}^{(n)}$, respectively.}

Now, we recall the way to derive auxiliary functions for the objective function.

\begin{lem}{(Auxiliary function \cite{fevotte2011algorithms}):}\label{Lemma1}
Given two functions $G(\mathbf{A}|\hat{\mathbf{A}})$ and $F(\mathbf{A})$, if $\forall \mathbf{A} \in \mathbb{R}$, there are $F(\mathbf{A}) = G(\mathbf{A}|{\mathbf{A}})$ and $F(\mathbf{A}) \le G(\mathbf{A}|\hat{\mathbf{A}})$.
The mapping $G(\mathbf{A}|\hat{\mathbf{A}})$ is an auxiliary function with respect to $F(\mathbf{A})$.
\end{lem}

\begin{myTheo}\label{Theory1}
Setting $\rho_{i_ni_{-n}} = \frac{\hat{\mathbf{A}}^{(n)}_{i_ni_{-n}}\mathbf{B}_{i_{n}i_{-n}}}{\sum_{i_{-n}} \hat{\mathbf{A}}^{(n)}_{i_ni_{-n}}\mathbf{B}_{i_{n}i_{-n}}}$. By the Lemma \ref{Lemma1}, we define an auxiliary function for the objective function of GWNTF with respect to $\mathbf{A}^{(n)}$ is givne by
\begin{equation}\label{AFGWNTF}
\begin{split}
 & G(\mathbf{A}^{(n)}, \hat{\mathbf{A}}^{(n)}) \\
 & = \sum_{n} \! \sum_{i_{\!n}\!,i_{\!-\!n}} \! \prod_{i_{\!n}\!,i_{\!-\!n}} \!\!\! \rho_{i_{\!n}i_{\!-\!n\!}} \! W_T\left(\mathbf{X}_{(n)}, (\mathbf{A}^{(n)} \mathbf{A}^{(-n)^T}) \oslash \boldsymbol{\rho} \right) \\
 & \quad + \mu \sum_{i_N,j_N} \mathbf{V}_{i_Nj_N} \left\| \mathbf{a}^{(N)}_{i_N} - \mathbf{a}^{(N)}_{j_N} \right\|^2,
\end{split}
\end{equation}
\end{myTheo}
\begin{proof}
 Since the objective function of GWNTF is convex with respect to $\mathbf{A}^{(n)}$, we use the property of convexity to derive the inequality $F(\mathbf{A}^{(n)} \circledast \boldsymbol{\rho}) \leq \sum_{ij}{\rho_{ij}} F(\mathbf{A}^{(n)}_{ij})$, which holds for all nonnegative elements of $\boldsymbol{\rho}$.
\end{proof}

\subsection{Update rule for $\mathcal{T}$}
We now derive the solution for the transport tensor $\mathcal{T}$ of the Wasserstein distance.
The smooth Wasserstein distance can be obtained through adding an extropic regularization term to the exact Wasserstein distance, and it can be viewed as desirable convex objective function with respect to the transport tensor $\mathcal{T}$.
Hence, we calculate first order derivative of \eqref{GWNTF} with respect to $\mathcal{T}$ is:
\begin{equation}\label{DerivaGWNTF}
\begin{split}
 & \frac{\partial W_T(\mathcal{X}, \hat{\mathcal{X}})}{\partial \mathbf{T}_{n_{i_nj_n}}^{i_{-n}}} = \mathbf{C}_{n_{i_nj_n}} + \frac{1}{\lambda} (\log \mathbf{T}_{n_{i_nj_n}}^{i_{-n}}+1) \\
 & \! + \! \alpha \! \left[\log(\Phi(\mathbf{T}_{\!n\!\!}^{i_{\!-\!n\!}}) \! \! \oslash \! \! \mathbf{X}_{\!(\!n\!)\!}^{i_{\!-\!n\!}})\right]_{i_{\!n}} \! \! \! + \! \beta \! \left[\!\log(\Psi(\mathbf{T}_{\!n\!}^{^{i_{\!-\!n\!}}}) \! \! \oslash \! \! \hat{\mathbf{X}}_{\!(\!n\!)\!}^{i_{\!-\!n\!}})\!\right]_{j_{\!n}}\!,\!
\end{split}
\end{equation}
where $\mathbf{X}_{(n)}^{i_{-n}}$ denotes the $i_{-n}$th column of mode-$n$ matricization of input tensor $\mathcal{X}$.
Then, let \eqref{DerivaGWNTF} equal to zero, we obtain the optimization of the transport tensor $\mathcal{T}$:
\begin{equation}\label{updateTGWNTF}
\begin{split}
 \mathbf{T}_{n_{i_nj_{n}}}^{i_{\!-\!n\!}} \! \! \! = & \! \left[ \mathbf{X}_{(n)}^{i_{\!-\!n\!}} \! \oslash \! \Phi(\mathbf{T}_{n}^{i_{\!-\!n\!}}) \right]^{\lambda\alpha}_{i_n\!} \! \! \left[ \! \hat{\mathbf{X}}_{(n)}^{i_{\!-\!n\!}} \! \oslash \! \Psi(\mathbf{T}_{n}^{i_{\!-\!n\!}}) \! \right]^{\lambda\beta}_{j_n\!} \\
 & \exp\left( -\lambda \mathbf{C}_{n_{i_nj_n}} -1 \right),
\end{split}
\end{equation}
where we denote $\mathbf{K}_n = \exp(-\lambda \mathbf{C}_{n}-1) \in \mathbb{R}_{\geq 0}^{I_n\times I_{n}}$, $\mathbf{u}_{i_{-n}} = \left( \mathbf{X}_{(n)}^{i_{-n}} \oslash \Phi(\mathbf{T}_{n}^{i_{-n}}) \right)^{\lambda\alpha} \in \mathbb{R}_{\geq 0}^{I_n}$, and $\mathbf{v}_{i_{-n}} =  \left( \hat{\mathbf{X}}_{(n)}^{i_{-n}} \oslash \Psi(\mathbf{T}_{n}^{i_{-n}}) \right)^{\lambda\beta} \in \mathbb{R}_{\geq 0}^{I_n}$.
Therefore, the solution of the sliced transport tensor can be represented as $\mathbf{T}_{n}^{i_{-n}} = \mbox{diag}(\mathbf{u}_{i_{-n}}) \mathbf{K}_n \mbox{diag}(\mathbf{v}_{i_{-n}})$.

The variable $\mathbf{u}_{i_{-n}}$ and $\mathbf{v}_{i_{-n}}$ both contain the unknown variable $\mathbf{T}_{n}^{i_{-n}}$, so we use the relation $\mathbf{T}_{n}^{i_{-n}} \mathbf{1} = \mbox{diag}(\mathbf{u}_{i_{-n}})\mathbf{K}_n \mathbf{v}_{i_{-n}}$ to derive the update rules of $\mathbf{u}_{i_{-n}}$ and $\mathbf{v}_{i_{-n}}$:
\begin{equation}\label{updateuGWNTF}
\begin{split}
& \mathbf{T}_{\!n}^{i_{\!-\!n\!}} \! \mathbf{1} \!\! = \! \mbox{diag}(\mathbf{u}_{i_{\!-\!n\!}})\mathbf{K}_n \! \mathbf{v}_{i_{\!-\!n\!}} \! \! \! = \! \! \left( \! {\mathbf{X}}_{\!(\!n\!)\!}^{i_{\!-\!n\!}} \! \oslash \! \Phi(\mathbf{T}_{\!n}^{i_{\!-\!n\!}}) \! \right)^{\!\lambda\alpha} \! \! \! \! \! \!\circledast \!\! \left( \mathbf{K}_{\!n} \mathbf{v}_{\!i_{\!-\!n\!}} \! \right) \\
&  \Rightarrow \quad \left( \mathbf{T}_{n}^{i_{-n}} \mathbf{1} \right)^{\lambda\alpha+1} = \left(\mathbf{X}_{(n)}^{i_{-n}} \right)^{\lambda\alpha} \circledast (\mathbf{K}_n \mathbf{v}_{i_{-n}})\\
&  \Rightarrow \quad \left[ \mathbf{u}_{i_{\!-\!n\!}} \! \circledast \! ( \mathbf{K}_n \mathbf{v}_{i_{\!-\!n\!}} ) \right]^{\lambda\alpha+1} \!\!\! = \! \left(\mathbf{X}_{(n)}^{i_{-n}} \right)^{\lambda\alpha} \! \!\! \circledast \! (\mathbf{K}_n \mathbf{v}_{i_{\!-\!n\!}}) \\
&  \Rightarrow \quad \mathbf{U}_n \! = \! [\mathbf{u}_1\!,\!\dots\!,\!\mathbf{u}_{I_{\!-\!n\!}}] \! = \! \left( \mathbf{X}_{(\!n\!)} \right)^{\!\frac{\lambda\alpha}{\lambda\alpha+1}} \! \oslash \! \left( \mathbf{K}_{\!n\!} \mathbf{V}_{\!n\!} \right)^{\frac{\lambda\alpha}{\lambda\alpha+1}}\!,\!
\end{split}
\end{equation}
Similarly, we have $\mathbf{V}_n = \left( \hat{\mathbf{X}}_{(n)} \right)^{\frac{\lambda\beta}{\lambda\beta+1}} \circledast \left( \mathbf{K}_n^T \mathbf{U}_n \right)^{\frac{-\lambda\beta}{\lambda\beta+1}}$.

\begin{algorithm}[t]
\caption{GWNTF}\label{algorithm}
\KwIn{$\mathcal{X} \in \mathbb{R}^{I_1\times \dots \times I_N}_{\geq 0}$, $\mathbf{C}_{n}, \forall n = 1,\dots,N$, rank number $R$, $\lambda$, $\alpha$, $\beta$ and $\mu$.}
\KwOut{$\mathbf{A}^{(n)}, n = 1,\dots,N$.}
\BlankLine
$\mathbf{K}_n = \exp(-\lambda \mathbf{C}_n - 1), n = 1,\dots,N$\;
Initialize $\mathbf{A}^{(n)}, n = 1,\dots,N$ randomly\;
$\phi = \frac{\lambda \alpha}{\lambda \alpha + 1}$, $\psi = \frac{\lambda \beta}{\lambda \beta + 1}$\;
\While{stopping criterion is not reached}{
\For{$n = 1,\dots,N$}{
$\mathbf{V}_n = \mbox{ones} \left(I_n, I_{-n}\right) \oslash I_n$\;
\For{k = 1,\dots,$\mbox{Sinkhorn Iteration}$}{
$\mathbf{V}_n = \hat{\mathbf{X}}_{(n)}^\psi \oslash \left( \mathbf{K}_n^T \left( \mathbf{X}_{(\!n\!)} \right)^{\!\phi} \! \oslash \! \left( \mathbf{K}_{\!n\!} \mathbf{V}_{\!n\!} \right)^{\phi} \! \right)^{\psi}$\!\;
}
$\mathbf{U}_n = \left( \mathbf{X}_{(\!n\!)} \right)^{\!\phi} \! \oslash \! \left( \mathbf{K}_{\!n\!} \mathbf{V}_{\!n\!} \right)^{\phi}$\;
$\Psi(\mathbf{T}_n) = \mathbf{V}_n \circledast (\mathbf{K}_n^T \mathbf{U}_n)$\;
}
\For{$n = 1,\dots,N-1$}{
${\mathbf{A}}^{(n)} = {\mathbf{A}}^{(n)} \circledast \frac{\Psi(\mathbf{T}_n) \! \oslash \! ( \mathbf{A}^{(n)} \mathbf{A}^{(\!-\!n\!)^T} ) \mathbf{A}^{(\!-\!n\!)}}{\mathbf{1} \mathbf{A}^{(\!-\!n\!)}}$\;
}
$\hat{\mathbf{A}}^{(N)} = \hat{\mathbf{A}}^{\!(\!N\!)\!} \! \circledast \! \frac{\beta\Psi(\mathbf{T}_N) \! \oslash \! ( \mathbf{A}^{(N)} \mathbf{A}^{(\!-\!N\!)^{\!T}} \! \! ) \mathbf{A}^{(\!-\!N\!)} \! + \! \mu \mathbf{D} \mathbf{A}^{\!(\!N\!)\!}}{ \beta \mathbf{1} \mathbf{A}^{(\!-\!N\!)} + \mu \mathbf{V} \mathbf{A}^{(N)}}$\;
}
\end{algorithm}

\subsection{Update rule for $\mathbf{A}^{(n)}$}
We derive the solution for the marginal distribution of target reconstructed tensor $[\![\mathbf{A}^{(1)}, \dots, \mathbf{A}^{(N)}]\!]$.
The objective function is not jointly convex with respect to $\mathbf{A}^{n}, \forall n = 1,\dots,N$.
Therefore we utilize the MM algorithm to seek the global minimum of objective function.

If $n \neq N$, and by Theorem \ref{Theory1}, we conduct auxiliary function for $\mathbf{A}^{(n)}$ as \eqref{AFGWNTF}.
We note that if we obtain the first order derivative of the above auxiliary function, we need to derive the partial derivative of smooth Wasserstein distance with respect to the target matrix $\hat{\mathbf{X}}_{(n)}$.
It is the same as the derivative of the KL divergence between transport matrix and target matrix with respect to the target matrix.
\begin{equation}\label{partialWTD}
\begin{split}
\frac{\partial W_T(\mathbf{X}_{(n)}, \hat{\mathbf{X}}_{(n)})}{\partial \hat{\mathbf{X}}_{(n)}} & = \beta \frac{\partial D_{\mbox{KL}}(\Psi(\mathbf{T}_{\!n\!}), \! \hat{\mathbf{X}}_{\!(\!n\!)})}{\partial \hat{\mathbf{X}}_{(n)}} \\
& = \beta \left(\mathbf{1} - \Psi(\mathbf{T}_{n})  \oslash \hat{\mathbf{X}}_{(n)}\right).
\end{split}
\end{equation}

By the chain rule of derivatives, we set the gradient of $G(\mathbf{A}^{(n)}, \hat{\mathbf{A}}^{(n)})$ to zero, and obtain:
\begin{equation}\label{gradientGWNTFAuxAn}
\begin{split}
\frac{\partial G(\mathbf{A}^{(n)}, \hat{\mathbf{A}})}{\partial \hat{\mathbf{A}}_{(n)_{i_nr}}} \! = \! \beta \left( \mathbf{1} \! - \! \Psi(\mathbf{T}_n) \! \oslash \! ( \mathbf{A}^{(n)} \mathbf{A}^{(\!-\!n\!)^T} ) \right) \! \mathbf{A}^{(\!-\!n\!)} \\
 = \underbrace{\beta \mathbf{1} \mathbf{A}^{(\!-\!n\!)}}_{\nabla_{\mathbf{A}^{(n)}}^{-} G(\mathbf{A}^{(n)})} - \underbrace{\beta \Psi(\mathbf{T}_n) \! \oslash \! ( \mathbf{A}^{(n)} \mathbf{A}^{(\!-\!n\!)^T} ) \mathbf{A}^{(\!-\!n\!)}}_{\nabla_{\mathbf{A}^{(n)}}^{+} G(\mathbf{A}^{(n)})}\!,
\end{split}
\end{equation}
Hence, the update of parameters $\mathbf{A}^{(n)}$ can be written as:
\begin{equation}\label{updateAn}
\begin{split}
\hat{\mathbf{A}}^{(n)}_{\mbox{MM}} & = {\mathbf{A}}^{(n)} \circledast \frac{\nabla_{\mathbf{A}^{(n)}}^{+} G(\mathbf{A}^{(n)})}{\nabla_{\mathbf{A}^{(n)}}^{-} G(\mathbf{A}^{(n)})} \\
& = {\mathbf{A}}^{(n)} \circledast \frac{\Psi(\mathbf{T}_n) \! \oslash \! ( \mathbf{A}^{(n)} \mathbf{A}^{(\!-\!n\!)^T} ) \mathbf{A}^{(\!-\!n\!)}}{\mathbf{1} \mathbf{A}^{(\!-\!n\!)}}.
\end{split}
\end{equation}

If $n = N$, in the similar way, we obtain the gradient of $G(\mathbf{A}^{(N)}, \hat{\mathbf{A}}^{(N)})$:
\begin{equation}\label{gradientGWNTFAuxAN}
\begin{split}
& \frac{\partial G(\mathbf{A}^{(\!N\!)} \! , \hat{\mathbf{A}}^{(\!N\!)}\!)}{\partial \hat{\mathbf{A}}_{(N)_{i_Nr}}} \! = \! \beta \! \left( \mathbf{1} \! - \! \Psi(\mathbf{T}_{\!N}) \! \oslash \! ( \mathbf{A}^{\!(\!N\!)\!} \mathbf{A}^{\!(\!-\!N\!)^T\!} ) \right) \! \mathbf{A}^{(\!-\!N\!)} \\
& \qquad\qquad\qquad\qquad + \mu \left( \mathbf{V} - \mathbf{D} \right) \mathbf{A}^{(N)}\\
& = \underbrace{ \!\! \left( \! \beta \mathbf{1} \mathbf{A}^{\!(\!-\!N\!)\!} \!\! + \!\! \mu \mathbf{V}\mathbf{A}^{\!(\!N\!)\!} \right) \! }_{\nabla_{\mathbf{A}^{(N)}}^{-} G(\mathbf{A}^{(N)})} \!\! - \!\! \underbrace{ \left( \! \beta \Psi(\mathbf{T}_{\!N\!}) \!\! \oslash \!\! ( \mathbf{A}^{\!(\!N\!)\!} \! \mathbf{A}^{\!(\!-\!N\!)\!^T\!}\! ) \! \mathbf{A}^{\!(\!-\!N\!)\!} \!\! + \!\! \mu \mathbf{D} \mathbf{A}^{\!(\!N\!)\!} \! \right) }_{\nabla_{\mathbf{A}^{(N)}}^{+} G(\mathbf{A}^{(N)})} \!,
\end{split}
\end{equation}
where $\mathbf{D}_{ii} = \sum_{j} \mathbf{V}_{ij}$ denotes the diagonal matrix.

Then, the update rule of $\mathbf{A}^{(N)}$ can be represented as:
\begin{equation}\label{updateAN}
\begin{split}
\hat{\mathbf{A}}^{(N)}_{\mbox{MM}} & = {\mathbf{A}}^{(N)} \circledast \frac{\nabla_{\mathbf{A}^{(N)}}^{+} G(\mathbf{A}^{(N)})}{\nabla_{\mathbf{A}^{(N)}}^{-} G(\mathbf{A}^{(N)})} \\
& = \hat{\mathbf{A}}^{\!(\!N\!)\!} \! \circledast \! \frac{\beta\Psi(\mathbf{T}_N) \! \oslash \! ( \mathbf{A}^{(N)} \mathbf{A}^{(\!-\!N\!)^{\!T}} \! \! ) \mathbf{A}^{(\!-\!N\!)} \! + \! \mu \mathbf{D} \mathbf{A}^{\!(\!N\!)\!}}{ \beta \mathbf{1} \mathbf{A}^{(\!-\!N\!)} + \mu \mathbf{V} \mathbf{A}^{(N)}}.
\end{split}
\end{equation}

Algorithm \ref{algorithm} recapitulates the update rules of GWNTF.
Matricization and vectorization of tensorial data leads to fast and simple implementations of proposed method.

\section{Comparison Framework and Results}
\label{section4}
In this section, we compare the proposed GWNTF with six dimensional reduction and clustering algorithms for image clustering tasks.

\subsection{Experimental settings.}

We used Columbia Object Image Library (COIL20) \cite{Nene1996} and  CMU PIE face database (PIE\_pose27) \cite{1004130} for our experiments.
The COIL20 data set consists of $1,440$ object images distributed over $20$ subjects, in which the dimension of each image is $32 \times 32$.
The PIE\_pose27 data set, a subset of the PIE face database, consisting of $2,856$ images.
It contains $42$ face images of $68$ people with different lighting conditions, and the size of all images is compressed to $32 \times 32$.
We reported the average results over $50$ Monte-Carlo runs.

\begin{itemize}
 \itemsep=0.0pt
  \item \textbf{Canonical K-means clustering method (Kmeans)} \cite{hartigan1979algorithm}: {We unfold each tensorial sample into a vector and apply them to $k$-means clustering directly.}
  \item \textbf{Nonnegative Matrix Factorization (NMF)} \cite{lee2000algorithms}:
      We unfold tensorial original data into a matrix to apply it to NMF algorithm. Then we obtain a low rank activation matrix and apply it to clutering tasks.
  \item \textbf{Graph Regularization Nonnegative Matrix Decomposition (GNMF)} \cite{cai2010graph}: We unfold tensorial data into a matrix and reduce its dimension into $K$ by GNMF, where the dictionary matrix is restricted in graph regularization term. The coefficient of regularization $\lambda = 100$, and the hyperparameter $p=5$.
  \item \textbf{Nonnegative Matrix Factorization with Sinkhorn Distance (SDNMF)} \cite{flamary2016optimal}:Similar to the above, we unfold the tensorial data into a matrix and apply it to SDNMF. The loss function between the unfolding tensor and reconstruction data is based on the Sinkhorn distance.
      We set the hyperparameters of regularizer $\lambda = 100$, $p = 5$, $\gamma = 1$, and $\epsilon = 10$.
  \item \textbf{Nonnegative Canonical Polyadic Decomposition (NCP)}  \cite{shashua2005non}:
      We implemented K-means clustering after factorizing the image set with NTF.
  \item \textbf{Graph Regularized Nonnegative Canonical Polyadic Decomposition (GNCP)} \cite{wang2011image}:
      GNCP takes the manifold structure of the data space in Nonnegative Canonical Polyadic and learns the high performance low-rank representations of tensorial data.
      The hyperparameter $\lambda = 10^4$, $p = 5$.
  \item \textbf{Graph Regularized Wasserstein nonnegative tensor factorization (GWNTF)}:
      We utilized Wassersterin distance to measure the original and reconstructed tensors, and combined the manifold structure of data space to improve the performance of low-rank representation.
      The hyperparameters $\alpha = \beta = 1$, $\mu = 10^4$, $\lambda = 100$, and $p = 5$.
\end{itemize}

\subsection{Main Results.}

\begin{table}[!t]
\vspace{-0.2cm}
\centering
\caption{The average Performance of COIL20.}\label{tab1}
\scalebox{0.83}{
\setlength{\tabcolsep}{2mm}{
\begin{tabular}{lc|c|c|c}
  \toprule[2pt]
  \multirow{7}{*}{}
    & \multicolumn{1}{c}{ACC(\%)} & \multicolumn{1}{c}{NMI(\%)} & \multicolumn{1}{c}{MI(\%)} & \multicolumn{1}{c}{Purity(\%)}\\
   \midrule[1pt]
    Kmeans & 54.03  & 72.54  & 70.48 & 57.99  \\
    NMF & 57.92  & 70.93  & 69.42 & 61.53 \\
    GNMF & 73.89  & 87.47  & 85.84 & 78.68 \\
    SDNMF & 78.82  & 87.34  & 86.62 & 81.18 \\
    NCP & 64.30  & 74.60  & 73.92 & 67.26 \\
    GNCP & 78.14  & 88.97  & 87.84 & 81.90 \\
    GWNTF & 78.23  & 90.52  & 88.90 & 82.67 \\
  \bottomrule[2pt]
\end{tabular}}
}
\vspace{-0.3cm}
\end{table}

\begin{table}[!t]
\centering
\caption{The average Performance of PIE\_pose27.}\label{tab2}
\scalebox{0.83}{
\setlength{\tabcolsep}{2mm}{
\begin{tabular}{lc|c|c|c}
  \toprule[2pt]
  \multirow{7}{*}{}
    & \multicolumn{1}{c}{ACC(\%)} & \multicolumn{1}{c}{NMI(\%)} & \multicolumn{1}{c}{MI(\%)} & \multicolumn{1}{c}{Purity(\%)}\\
   \midrule[1pt]
    Kmeans & 24.58  & 49.33  & 48.34 & 26.75  \\
    NMF & 29.20  & 56.21  & 55.38 & 31.34 \\
    GNMF & 77.38  & 92.49  & 91.25 & 81.90 \\
    SDNMF & 74.23  & 91.30  & 89.73 & 79.94 \\
    NCP & 34.26  & 59.42  & 58.41 & 36.90 \\
    GNCP & 76.01  & 92.25  & 90.98 & 81.42 \\
    GWNTF & 78.26  & 92.97  & 91.69 & 82.78 \\
  \bottomrule[2pt]
\end{tabular}}
}
\vspace{-0.3cm}
\end{table}

Table \ref{tab1} and table \ref{tab2} show the results of four evaluations, such as Accuracy (ACC), Normalized Mutual Information (NMI), Mutual Information (MI), and Purity, on COIL20 and PIE\_pose27 data set, respectively.
From the table, we see that the clustering results of GWNTF is competitive with that of SDNMF in COIL20. We note that NMI and MI of GWNTF are $3\%$ and $2\%$ higher than that of SDNMF, respectively, which means the structure information of each mode is necessary for low-rank representation. The purity of GWNTF is the highest compared with the others.

From the table \ref{tab2}, we see that the clustering accuracy of GWNTF is more than $4$\% absolutely higher than that of the SDNMF in PIE\_pose27.
It can be noted that all evaluations of GWNTF obtain the best performance than other six algorithms.
It demonstrates the efficiency of our algorithm.

\section{Conclusion}
\label{section5}

In this paper, we proposed a Graph regularized Wasserstein Nonnegative Tensor Factorization, GWNTF, which exploits manifold structure information of data space and features correlation.
We utilized Maximization-Minimization (MM) algorithms to relax the objective function with respect to the latent factors and optimize them.
By this approximation, GWNTF outperforms NMF and NTF based algorithms in the clustering tasks with two image data sets.
Although GWNTF performs well, it still limits the applicability to large-scale or high-order data.
We will focus on the efficient calculations of the transport tensor in the future.

\small
\bibliographystyle{IEEEbib}
\bibliography{strings,refs}

\end{document}